\documentclass[conference]{IEEEtran}
\IEEEoverridecommandlockouts

\usepackage{tikz}

\def\BibTeX{{\rm B\kern-.05em{\sc i\kern-.025em b}\kern-.08em
    T\kern-.1667em\lower.7ex\hbox{E}\kern-.125emX}}
\begin{document}

\title{Hereditary Geometric Meta-RL: \\Nonlocal Generalization via Task Symmetries
\thanks{$^{1}$ P. Nitschke is with the Paulson School Of Engineering And Applied Sciences, Harvard University, Cambridge, US. (\href{mailto:paul.nitschke@outlook.de}{paul.nitschke@outlook.de})

$^{2}$ S. Talebi is with the UCLA Samueli School of Engineering and Applied Science, University of California, Los Angeles, US. (\href{mailto:s.talebi@ucla.edu}{s.talebi@ucla.edu})}
}

\author{
\IEEEauthorblockN{Paul Nitschke$^{1}$}
\and
\IEEEauthorblockN{Shahriar Talebi$^{2}$}
}

\maketitle

\begin{abstract}
Meta-Reinforcement Learning (Meta-RL) commonly generalizes via smoothness in the task encoding. While this enables local generalization around each training task, it requires dense coverage of the task space and leaves richer task space structure untapped. In response, we develop a geometric perspective that endows the task space with a  ``hereditary geometry'' induced by the inherent symmetries of the underlying system. Concretely, the agent reuses a policy learned at the train time by transforming states and actions through actions of a Lie group. This converts Meta-RL into symmetry discovery rather than smooth extrapolation, enabling the agent to generalize to wider regions of the task space. We show that when the task space is inherited from the symmetries of the underlying system, the task space embeds into a subgroup of those symmetries whose actions are linearizable, connected, and compact—properties that enable efficient learning and inference at the test time. To learn these structures, we develop a differential symmetry discovery method. This collapses functional invariance constraints and thereby improves numerical stability and sample efficiency over functional approaches. Empirically, on a two-dimensional navigation task, our method efficiently recovers the ground-truth symmetry and generalizes across the entire task space, while a common baseline generalizes only near training tasks.
\end{abstract}


\section{Introduction}

Reinforcement Learning (RL) continues to face fundamental challenges, particularly in terms of generalization and sample efficiency \cite{kirk_survey_2023}. Meta-RL addresses this by training the agent on multiple tasks sampled from a task space $\mathbb{M}$ to generalize to ``similar'' unseen tasks \cite{wang_learning_2017}. 
A prevalent approach is memory-based Meta-RL \cite{rakelly_efficient_2019} which posits a smooth-manifold hypothesis on $\mathbb{M}$: Jointly learn a task encoder $\psi: \mathbb{M} \to \taskSpace\subseteq\mathbb{R}^d$ and a joint policy for all tasks conditioned on $\taskembedding \in \taskSpace$. While plausible, memory based agents typically generalize only \emph{locally} around the training tasks \cite{mandi_effectiveness_2022} -- thereby requiring a dense coverage of $\mathbb{M}$ with training tasks to generalize successfully.

We argue that memory-based methods generalize locally because they \textit{generalize via smoothness} in $\taskSpace$. The encoder $\psi$ commonly maximizes the mutual information (MI) between task trajectories and embeddings \cite{li_towards_2024}, either via contrastive learning (CL) \cite{li_focal_2020} or variational auto-encoders \cite{zintgraf_varibad_2021}. The MI objective is insensitive to the global geometry of $\taskSpace$. Specifically, the unique global optima of CL are equidistant embeddings, where all training-task encodings lie equally far apart \cite{graf_dissecting_2021}. This construction preserves only tangential geometry and thereby yields faithful encodings for precisely those tasks that are infinitesimally close to the training set, while potentially neglecting richer structure present in $\mathbb{M}$.
In practice, departures from CL's idealized regime—e.g., smaller embedding dimension or fewer negatives samples—can induce limited non-tangential structure \cite{damrich_umap_2021, bohm_attraction-repulsion_2022, lee_parameterizing_2023}. However, these effects are challenging to tune and remain fundamentally local.
Likewise, standard actor–critic backbones such as PPO and SAC \cite{schulman_proximal_nodate, haarnoja_soft_2018} for policy optimization exhibit primarily local generalization around training tasks \cite{kirk_survey_2023}, also leaving broader task-space transfer unresolved. This motivates our central question:

\textit{Can we endow $\mathbb{M}$ with a structure richer than the smooth manifold hypothesis that enables generalization beyond local smoothness?}

In this work, we introduce a framework that captures symmetry-induced geometry in the task representation $\taskSpace$. Inspired by biological agents and their case-based reasoning approach to generalization \cite{pal_soft_2000}, we propose \emph{retrieve \& reuse} to generalize non-locally: Given a test task, retrieve a similar training task and roll out its policy after transforming states and actions via left actions of a Lie group $\Gnormal$--for an introduction to Lie Groups we refer to \cite{lee_introduction_2012} and also lectures \cite{schuller_geometric_2015}. During training, the agent discovers $\Gnormal$, while the specific task inference $\gnormal \in \Gnormal$ is performed at test time. This enables efficient knowledge transfer across wider regions of the task space $\mathbb{M}$. 
Our main contributions are as follows: 
\begin{itemize}
    \item After formalizing the problem setup (\S\ref{sec:problem},) we propose our central symmetry hypothesis for Meta-RL, called a \emph{hereditary geometry} (\S\ref{sec:hereditary_geometry}.)
    \item We identify a salient regime in which the geometry of $\taskSpace$ is inherited from the symmetries of the system. This endows $\mathbb{M}$ with a hereditary geometry and subsumes many settings where tasks are intuitively perceived as ``similar'' (\S\ref{sec:hereditary_geometries_via_symmetries}.)
    \item We formulate hereditary geometry discovery as a concrete learning problem, estimating inherent symmetries from trajectory data and performing inference over group actions (\S\ref{sec:learning_problem}.)
\end{itemize}
Finally, we empirically validate our approach on a 2-D navigation benchmark in \S\ref{sec:simulations} and compare against existing methods. We provide concluding remarks in \S\ref{sec:conclusions}.

\section{Problem Formulation}\label{sec:problem}
Herein, we first define the standard Meta-RL problem setup adapted from \cite{beck_survey_2024}. A task space $\mathbb{M}$ is a set of elements $\mathcal{M} \in \mathbb{M}$, each representing a Markov decision process $\mathcal{M}=\{S,A,R_\mathcal{M},T_\mathcal{M},\gamma\}$. All tasks share the same state and action spaces $S$ and $A$ and discount rate $\gamma \in (0,1)$, but differ in their reward $R_\mathcal{M}: S \times A \times S \rightarrow [0,1]$ and transition functions $T_\mathcal{M}: S \times A \rightarrow \Delta_S$, where $\Delta_S$ denotes the probability simplex on $S$. Then, the Meta-RL problem proceeds in two stages: At the \textit{meta-train} time, the agent is given black-box access to $N_\text{train}+1$ uniformly sampled tasks $\mathcal{M}_0,...,\mathcal{M}_{N_\text{train}} \sim \mathcal{U}(\mathbb{M})$ to train a policy $\pi$ which aims to minimize the mean incurred regret $\mathcal{R}$ at the \textit{meta-test} time
\begin{multline}\label{equ:metarl_regret}
    \argmin_{\pi \in \Pi} \mathbb{E} \bigg[
    \sum_{h=K}^{H} \mathbb{E}_{\tau_{:h} \sim (\pi(\tau_{:h-1}), \;\mathcal{M})}[\mathcal{R}\left(\pi(\tau_{:h}\right); \mathcal{M})]\bigg],
\end{multline}
where the outer expectation is over $[\mathcal{M} \sim \mathcal{U}(\mathbb{M})]$, $H\in \mathbb{N}$ denotes the number of episodes, $\tau_{:h}$ all trajectories collected up to episode $h$. The policy $\pi(\tau_{:h}) = \pi(\tau_{:h} ; \mathcal{M})$ can, in general, be non-Markovian, that is depend on all past trajectories $\tau_{:h-1}$, and generally depends on the latent task $\mathcal{M}$.
The value function associated with the policy $\pi(\tau_{:h}; \mathcal{M})$ in the task $\mathcal{M}$ is
\begin{equation}\label{equ:value_function}
    V^{\pi(\tau_{:h} ; \mathcal{M})}_\mathcal{M}(s_0) \coloneqq \mathbb{E} \left[ \sum_{t=0}^\infty \gamma^t R_\mathcal{M}(s_t,a_t, s_{t+1}) \; \bigg| \; s_0 = s_0 \right],
\end{equation}
where the expectation is over $[a_t \sim \pi(a_t \mid s_t, \tau_{:h}; \mathcal{M}), s_{t+1} \sim T_\mathcal{M}(s_{t+1} \mid s_t, a_t)]$, and $V^*_{\mathcal{M}}(s_0)$ denotes the value function of a policy that maximizes \eqref{equ:value_function}, both evaluated at the initial state $s_0 \in S$. Then, the regret associated with the policy $\pi$ in the task $\mathcal{M}$ is defined as $\mathcal{R}(\pi; \mathcal{M}) \coloneqq V^*_\mathcal{M}(s_0)-V^{\pi}_\mathcal{M}(s_0)$. The set $\Pi$ represents a family of admissible policies, usually a dense neural network, and the ``shot parameter'' $K \in \mathbb{N}$ controls the number of free exploration episodes. 

As $\pi(\tau_{:h}; \mathcal{M})$ generally depends on $\mathcal{M}$, a successful meta-agent must both identify the new test task and quickly generalize its knowledge from the training tasks to roll-out an optimal policy in its belief of the test task. While these two objectives may generally interfere, this work focuses on the latter; that is, the regime in which the generalization performance dominates. We therefore set the shot parameter sufficiently high to allow one to obtain a relatively reliable estimate of the task before incurring any regret.

A common formalization of Meta-RL is \emph{memory-based Meta-RL} \cite{doshi-velez_hidden_2013, hallak_contextual_2015, rakelly_efficient_2019} which endows $\mathbb{M}$ with a manifold hypothesis:
For every $\mathcal{M} \in \mathbb{M}$, there exists an unknown task encoding $\taskembedding \in \mathbb{R}^d$ such that for all $s,~s' \in S, ~a \in A$
\begin{align*}
    R_\mathcal{M}(s,a) &= R(s,a;\taskembedding) \eqqcolon R_\taskembedding(s,a), \\
    T_\mathcal{M}(s'|s,a) &= T(s'|s,a;\taskembedding) \eqqcolon T_\taskembedding(s'|s,a).
\end{align*}
We denote the set of all task encodings by $\taskSpace \ni z$.

\begin{figure}[!t]
\centering
\begin{tikzpicture}[scale=0.8,>=stealth]
  \def\R{2.0} %

    \draw[gray!60,->] (-2.6,0)--(2.6,0);
    \draw[gray!60,->] (0,-2.6)--(0,2.6);
    
    \draw[gray!80,densely dotted] (0,0) circle (\R);
    
    \draw[black,line width=1.2pt] (-0.12,-0.12)--(0.12,0.12)
                               (-0.12,0.12)--(0.12,-0.12);
    \node[gray!50!black] at (0.35,-0.18) {$s_0$};

    \node at (2,0) [blue, rectangle,draw] (m0) {};
    \node at (2.4,-0.2) [blue] {$z_0$};

    \node at (0,2) [red, rectangle,draw] (m0) {};
    \node at (0.4,1.8) [red] {$z_1$};

    \node at (-1.4,-1.4) [ForestGreen, rectangle,draw] (m0) {};
    \node at (-1,-1.2) [ForestGreen] {$z_2$};
\end{tikzpicture}
\caption{Illustration of the $2$-D navigation task. After learning to navigate from the origin $s_0$ to the goal positions \textcolor{blue}{$\taskembedding_0$} and \textcolor{red}{$\taskembedding_1$}, the agent aims to generalize their knowledge to navigate to the unseen location \textcolor{ForestGreen}{$\taskembedding_2$} at the test time.
}
\label{fig:circle_navigation}
\end{figure}
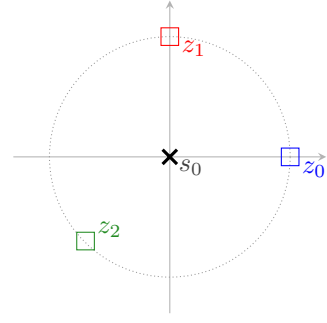

Given black-box access to training tasks $\mathcal{M}_0,...,\mathcal{M}_{N_\text{train}} \in \mathbb{M}$, memory-based approaches jointly train a task encoder $\psi$ and a global policy $\pi\big(a \mid s;z\big)\nonumber.$ The policy conditions on the current state $s$ and the task embedding $\taskembedding$ of the current task $\mathcal{M}$.

\begin{example}[$2$-D Navigation]\label{example:2_d_navigation}
    A popular Meta-RL example is \textit{$2$-D navigation} \cite{rakelly_efficient_2019, lee_improving_2021}: Starting from the origin $s_0 = 0$, the agent aims to navigate to different unknown goal positions $\taskembedding \in S^1 \coloneqq \taskSpace$ on the unit circle (cf. Figure \ref{fig:circle_navigation}) where the reward and transition function are given by
    \begin{align}
        R(s,a; \taskembedding) = -||s-\taskembedding||_2^2, && 
        T(s' \mid s,a; \taskembedding) = s+a.\nonumber
    \end{align}
    Given a step-size $\alpha>0$ (which we set to $1$ in the following for notational simplicity and to $\alpha=0.1$ in the sample implementation), the optimal policy of a task $\taskembedding$ is taking a step towards the goal location $\taskembedding$: $\pi^*(a \mid s; \taskembedding) = \alpha (\taskembedding-s)$.
\end{example}

\begin{assumption}
    We assume that reward $R(s,a;\taskembedding)$, transition $T(s' | s,a;\taskembedding)$, and optimal policy $\pi(a|s;z)$ are Lipschitz continuous in $s,a$ and $\taskembedding$ where the Lipschitz continuities in $R$ and $T$ are respectively quantified by the $L_2$ norm and the Wasserstein distance while task-distances are measured in $\taskSpace$ and denoted by $d(\mathcal{M},\mathcal{M}')$. 
\end{assumption}
This is a standard assumption in the Meta-RL literature \cite{fu_performance_2022} which reads that similar tasks (measured in $\taskSpace$) have similar optimal policies, see \cite{fu_performance_2022} for further discussion. 

Exploiting the Lipschitz continuity of the policy in $\taskembedding$, an agent faithfully encoding $\mathbb{M}$ can then generalize their policy to test tasks that are ``Lipschitz close'' to a training task in $\taskSpace$. More precisely, for every training task $\mathcal{M}_i, ~i \in \{0,...,N_\text{train}\},$ there exists some constant $\epsilon_i > 0$ such that the agent generalizes to $\mathcal{M} \in \mathbb{M}$ if $d(\mathcal{M}, \mathcal{M}_i) < \epsilon_i$ for some training task $\mathcal{M}_i$. 

While possible, such an approach is highly inefficient: It requires the training tasks to densely cover $\mathbb{M}$ and, as a result, discards potentially global structures of the task space $\mathbb{M}$ by replacing it with a purely local approximation.

In response, we pose the following problem: Train an agent to expand its generalization beyond local smoothness to \textit{non-local test tasks}, that is, test tasks that are not necessarily close to the training set in terms of their Lipschitz continuity captured by each $\epsilon_i$.

\textbf{Problem statement.}
    Consider a black-box access to training tasks $\mathcal{M}_0,...,\mathcal{M}_{N_\text{train}} \in \mathbb{M}$ that do not (necessarily) densely cover $\mathbb{M}$
    \begin{align}
        \exists \mathcal{M} \in \mathbb{M} \text{ s.t. } d(\mathcal{M},\mathcal{M}_i) > \epsilon_i && \forall i \in \{1,...,N_\text{train}\}, \nonumber
    \end{align}
    where $\epsilon_i>0$ represents the prior generalization constant.
    Learn a policy $\pi$ that generalizes non-locally, and uniformly within $\mathbb{M}$
    \begin{align}
        \mathcal{R}(\pi; \mathcal{M}) < \epsilon && \forall \mathcal{M} \in \mathbb{M}, \nonumber
    \end{align}
for a uniform constant $\epsilon>0$ comparable with $\epsilon_i$ (up to the Lipschitz continuity factor of $\pi$ in $z$.)

While these requirements are harder to satisfy, they promote utilization of richer structures in $\mathbb{M}$ beyond Lipschitz continuity, which current memory-based settings do not induce. In this work we endow $\mathbb{M}$ with a \textit{geometric hypothesis} that enables such non-local generalization. Then, our goals are threefold: (i) To formalize this geometric hypothesis; (ii) To show why we expect common Meta-RL applications to enjoy such geometric structure; and (iii) To learn such a policy purely from trajectory data samples.

\section{Hereditary Geometries}\label{sec:hereditary_geometry}

Biological agents commonly generalize by first \textit{retrieving} a similar, known situation and then \textit{reusing} it, also called case-based-reasoning \cite{pal_soft_2000}. For example, an ice skater can rollerblade by first recalling their ice-skating experience, then reusing it by applying the same movements but on wheels and asphalt rather than blades and ice. We formalize this intuition by positing that the optimal policy $\pi^*(a \mid s; \taskembedding)$ in the test task coincides with an optimal policy of some training task $\mathcal{M}_0$ after transforming $S$ and $A$ via left actions $L_\gnormal: S \rightarrow S$ and $K_\gnormal: A \rightarrow A$ of a Lie group $\Gnormal$:
\begin{multline}
\label{equ:hereditary_geometry}
    \forall \taskembedding \in \taskSpace \; \exists \;  \gnormal \in \Gnormal\; \text{s.t.}\\
    \pi^*(a \mid s; \taskembedding) = K_\gnormal^{-1}\left(\pi^*(a \mid L_\gnormal \cdot s; \taskembedding_0)\right), \; \forall s \in S, a \in A.
\end{multline}
Here, the action $K_\gnormal$ on the distribution $\pi^*$ over $A$ is defined via inversion (see Equation 12 in \cite{van_der_pol_mdp_2020} for a similar definition regarding equivariant policies) and we call $\mathcal{M}_0$ the ``base task.'' For instance, we have $\pi_0 = \pi_\text{ice skating}$, $L_\gnormal\{\text{asphalt ground}\} \rightarrow \{\text{icy ground}\}$ and $K_\gnormal(a) = a$ in the ice skater example. Assuming a group structure in the transformations $L$ and $K$ allows the agents to efficiently generalize at test time by inferring $\Gnormal$, $L_\gnormal$ and $K_\gnormal$ at train-time and only inferring the respective group element $\gnormal \in \Gnormal$ at test-time.

Equation \eqref{equ:hereditary_geometry} converts the conventional meta-RL optimization problem into a supervised symmetry discovery problem. To this end, we make a standard assumption from symmetry discovery \cite{benton_learning_2020, yang_latent_2024} that $\Gnormal$'s left actions can be \textit{linearized}: 

\begin{definition}[Linear left actions]
    The left actions $L_\gnormal$ and $K_\gnormal$ are called \textit{linear} if there exist diffeomorphisms $\encoderS: S \rightarrow \tilde{S}$ and $\encoderA: A \rightarrow \tilde{A}$ such that for all $\gnormal \in \Gnormal$
    \begin{align}
        \encoderS \circ L_\gnormal \circ \encoderS^{-1} &\coloneqq \tilde{L}_\gnormal \in \text{GL}^+(| S|, \mathbb{R})\nonumber,\\
        \encoderA \circ K_\gnormal \circ \encoderA^{-1} &\coloneqq \tilde{K}_\gnormal \in \text{GL}^+(| A|, \mathbb{R})\nonumber,
    \end{align}
    where $\text{GL}^+(d, \mathbb{R})$ denotes the general linear group of $d$-dimensional, real-valued matrices with strictly positive determinant. Then, we also call $\encoderS$ and $\encoderA$ \textit{representations} of $L_\gnormal$ and $K_\gnormal$.\label{prop:g_normal_is_linear}
\end{definition}
Linearization allows a simple parametrization of $L_\gnormal$ and $K_\gnormal$ in the learning problem and always holds true locally under regularity assumptions by the \textit{rank theorem} (see Theorem 4.12 in \cite{lee_introduction_2012}).

Merging the two above conditions yields our geometric hypothesis:
\begin{definition}[\textbf{Hereditary Geometry}]
    We call the geometry of $\mathbb{M}$ \emph{hereditary} if there exists a task encoding $\taskSpace$ and a Lie group $\Gnormal$ with linear left actions $L_\gnormal$ and $K_\gnormal$ that fulfill \eqref{equ:hereditary_geometry}. 
\end{definition}

\noindent A property of an object is called hereditary if it is inherited by all of its subobjects. Conjugating the left-action in Equation \eqref{equ:hereditary_geometry} reads that the geometry in $\mathbb{M}$ is inherited by any task in $\mathbb{M}$, implying that the geometry in $\mathbb{M}$ is hereditary--hence the name.

Next, we show that the hereditary geometry in $\pi^*$ naturally traces back to $R$ and $T$. This model based approach lays the ground for our learning problem in \S\ref{sec:learning_problem}.
\begin{lemma}\label{lemma:pushforward_mdp_implies_pushforward_optimal_policies}
    Assume there exists a Lie group $\Gnormal$ with linear left actions $L_\gnormal: S \rightarrow S$ and $K_\gnormal: A \rightarrow A$ such that
    \begin{multline}\label{equ:push_forward_reward_and_transition}
        \forall \; \taskembedding \in \taskSpace \; \exists \; \gnormal \in \Gnormal \text{~s.t.~}\\
        \left\{
        \begin{aligned}
            &R(s,a;\taskembedding) = R(L_\gnormal \cdot s, K_\gnormal \cdot a;\taskembedding_0) \\
            &T(s' \mid s,a;\taskembedding) = T( L_\gnormal \cdot s' \mid L_\gnormal \cdot s, K_\gnormal \cdot a;\taskembedding_0)
        \end{aligned}\right.,
    \end{multline}
    for all $s,s' \in S, a \in A$.
    Then, the geometry in $\mathbb{M}$ is hereditary.
\end{lemma}

\begin{proof}
We first show that the hereditary geometry in $R$ and $T$ translates into a hereditary geometry in the optimal Q-function $Q^*(s,a; \taskembedding)$ which is the unique fixed point of the Bellman operator $\mathbb{B}:\mathbb{R}^{|S| \times |A| \times |\taskSpace|} \rightarrow \mathbb{R}^{|S| \times |A| \times |\taskSpace|}$
\begin{multline}
    (\mathbb{B}Q)(s,a; \taskembedding) \coloneqq \mathbb{E}_{s'}
    \bigg[ R(s,a; \taskembedding) + \gamma \max_{a' \in A} Q(s',a'; \taskembedding)\bigg],\nonumber
\end{multline}
for $s \in S, a \in A$ and $\taskembedding \in \taskSpace$ and the expectation is taken over $[s' \sim T(s'|s,a; \taskembedding)]$. Now, recall value iteration (VI) which recursively defines $Q^{t+1}(s,a, \taskembedding) \coloneqq (\mathbb{B}Q^t)(s,a, \taskembedding)$ and $Q^0(s,a, \taskembedding) = 0$ for all $s \in S, a \in A,\taskembedding \in \taskSpace$ and converges to the optimal Q-function: $\lim_{t \rightarrow \infty}Q^t(s,a,\taskembedding) = Q^*(s,a, \taskembedding)$ \cite{sutton_reinforcement_2018}. Then, we show via induction over $t \in \mathbb{N}$ that the claim holds for every iterate of VI. Taking the limit in $t$ on both sides concludes the claim for $Q^*$. Finally, the lemma follows by defining $\pi^*(a | s; \taskembedding)$ as the greedy policy. \qed
\end{proof}

\begin{example}\label{example:2_d_navigation_can_be_pushed_forward}
    A hereditary geometry $\Gnormal$ in the $2$-D navigation Example \ref{example:2_d_navigation} is given by
    \begin{align}
        \Gnormal = SO(2, \mathbb{R}) && 
        L_\gnormal = g^{-1} \cdot s && K_\gnormal = g^{-1} \cdot a,\nonumber
    \end{align}
    where $\text{SO}(2, \mathbb{R})$ is the two-dimensional special-orthogonal group which acts on states and actions via standard matrix multiplication.
\end{example}
\begin{proof}
Define $\taskembedding_0 = (1,0)^T$ as the base encoding. As $\taskSpace = S^1$, each element $\taskembedding \in \taskSpace$ can be written as $\taskembedding = B \cdot \taskembedding_0$ for some $B \in SO(2, \mathbb{R})$.
Then, we have for the reward function $R$
\begin{multline}
    R(s,a;\taskembedding)
    = \left( 
        s
        - 
        B \cdot \taskembedding_0
    \right)^T
    \left( 
        s- B \cdot \taskembedding_0
    \right)
    \\
    = \left( 
        s
        - 
        B \cdot \taskembedding_0
    \right)^T
    BB^T
    \left( 
        s
        - 
        B \cdot \taskembedding_0
    \right)\\
    = \left( B^T \cdot s
        - B^T
        B \cdot \taskembedding_0
    \right)^T
    \left( B^T \cdot s
        - B^T
        B \cdot \taskembedding_0
    \right)\\
    = (L_\gnormal \cdot s - \taskembedding_0)^T(L_\gnormal \cdot s-\taskembedding_0) = R(L_\gnormal \cdot s, K_\gnormal \cdot a; \taskembedding_0)\nonumber
\end{multline}
as $B^T=B^{-1}$ for $B \in \text{SO}(2,\mathbb{R})$ and similarly for $T$.\qed

\end{proof}

\section{Hereditary Geometries via Symmetries}\label{sec:hereditary_geometries_via_symmetries}

Observe the high-level structure in the preceding example: After endowing $\taskSpace$ with the geometry $\text{SO}(2, \mathbb{R})$, we embedded $\text{SO}(2, \mathbb{R})$ into $S$ and $A$ while leaving $R$ and $T$ invariant, that is, we embed into the symmetries of the base task. We now formalize this intuition and subsequently show in Theorem \ref{theo:her_geo_from_symmetry} that the geometry in $\mathbb{M}$ is then hereditary. This provides one general, concrete setting where $\mathbb{M}$ has a hereditary geometry, namely if the geometry of the task space arises from the symmetries of the system.

We first define the geometry in the task space and the symmetries of a task, starting with the former.

\begin{assumption}[Geometric task space]\label{ass:geometry_hypothesis}
    There exists a compact and connected Lie group $\Gtask$ acting on a base task encoding $\taskembedding_0 \in \taskSpace$ with a left action $J_\gtask: \taskSpace \rightarrow \taskSpace$ that spans $\taskSpace$
    \begin{align}
        \bigcup_{\gtask \in \Gtask} \{J_\gtask \cdot \taskembedding_0\} = \taskSpace.\nonumber
    \end{align}
\end{assumption}

Combining the memory-based setting with Assumption \ref{ass:geometry_hypothesis} implies that the entire task space $\mathbb{M}$ collapses to the tuple $(\mathcal{M}_0, \Gtask)$ which we call a \textit{geometric meta-MDP}
\begin{multline}
    \bigcup_{\gnormal \in \Gnormal} \{(S,A,R_{J_\gnormal \cdot \taskembedding_0}, T_{J_\gnormal \cdot \taskembedding_0}, \gamma)\} = \bigcup_{\taskembedding \in \taskSpace} \{(S,A,R_\taskembedding, T_\taskembedding, \gamma)\} \\= \bigcup_{\mathcal{M} \in \mathbb{M}} \{(S,A,R_\mathcal{M}, T_\mathcal{M}, \gamma)\} = \mathbb{M}.\nonumber
\end{multline}

We now define the symmetries of a single task. Meta-RL applications commonly enjoy rich, high-dimensional symmetries as they often originate in physical systems, such as robotics, that naturally exhibit strong symmetries.\footnote{Or as Philip Anderson puts it: \textit{It is only slightly overstating the case to say that physics is the study of symmetry} \cite{anderson_more_1972}.} A task $\mathcal{M}$ is called symmetric with respect to a compact and connected Lie group $\Genv$ if $\Genv$ leaves $\mathcal{M}$ invariant \cite{van_der_pol_mdp_2020}:

\begin{definition}[Symmetric MDP]
    Let $\Genv$ act on $S$ and $A$ via linear left actions $L_\genv: S \rightarrow S$ and $K_\genv: A \rightarrow A$. Then, the tuple $(\mathcal{M}, \Genv)$ is called a \textit{symmetric MDP} if $R$ and $T$ are invariant with respect to $\Genv$ for all $s,s' \in S, a \in A$
    \begin{align}
        R(L_\genv \cdot s, K_\genv \cdot a; \taskembedding) &= R(s,a; \taskembedding), \label{equ:invariance_R}\\
        T(L_\genv \cdot s' \mid L_\genv \cdot s,K_\genv \cdot a; \taskembedding) &= T(s' \mid s,a; \taskembedding). \label{equ:equivariance_T}
    \end{align}
\end{definition}

\begin{example}[Symmetry in $2$-D navigation]\label{example:symmetry_in_2_D_navigation}
    A symmetry in each navigation task from Example \ref{example:2_d_navigation} is given by
    \begin{align}
        \Genv &= \text{SO}(2, \mathbb{R}), && \encoderS(s) = s-\taskembedding, && \encoderA(a) = a \nonumber,\\
        L_\gnormal \cdot s &= \gnormal \cdot s, && K_\gnormal \cdot a = \gnormal \cdot a, && \nonumber
    \end{align}
    as the two-norm in the reward function is rotation invariant and the transition function is linear.
\end{example}

In a symmetric MDP, the $\Genv$-invariance of $R$ and $T$ translates into an equivariance in $\mathcal{M}$'s optimal policies \cite{van_der_pol_mdp_2020}:
\begin{lemma}
    \textit{(Symmetric MDPs admit equivariant optimal policies)}
    Let $(\mathcal{M}, \Genv)$ denote a symmetric MDP. Then, its optimal policy $\pi^*$ fulfills an equivariance property with respect to $L_\genv$ and $K_\genv$: For all $\genv \in \Genv, s \in S, a \in A$
    \begin{align}
        \pi^*(a \mid L_{\genv} \cdot s; \taskembedding) = K_{\genv}^{-1}\cdot \pi^*(a \mid s; \taskembedding),\label{equ:symmetry_equivariance}
    \end{align}
    i.e., the pushforward of the measure $\pi^*(. \mid s; \taskembedding)$ by $K_n$.
\end{lemma}

We now formalize the embedding of $\Gtask$ into $\Genv$ from the beginning of \S\ref{sec:hereditary_geometries_via_symmetries} which, in turn, enables endowing $\mathbb{M}$ with a hereditary geometry in Theorem \ref{theo:her_geo_from_symmetry}. To make the embedding independent of the chosen base task, we require that all tasks share compatible symmetries:
\begin{definition}[Compatible Symmetry]\label{def:compatible-Symmetry}
    We say that $\Genv$ is \textit{compatible with the symmetries} of a geometric Meta-MDP $(\mathcal{M}_0, \Gtask)$ if for any other $\mathcal{M} \in \mathbb{M}$ induced by $\gtask \in \Gtask$, $\Genv$ is a symmetry with representation $(\encoderS_\gtask, \encoderA_\gtask)$ which changes equivariantly in $\Genv$; that is, $\Genv$ induces linear left actions $J_\genv: S \rightarrow S$ and $\tilde{J}_\genv: \tilde{S} \rightarrow \tilde{S}$ such that
    \begin{enumerate}
        \item denoting the representation of $\Genv$ in $\mathcal{M}_0$ by $(\encoderS_0, \encoderA_0)$,
    \begin{align}
    \exists \genv \in \Genv:   \encoderS_\gtask(s) = \tilde{J}_{\genv} \circ \encoderS_0 \circ J_{\genv}^{-1} (s) && \forall s \in S,\label{equ:encoder_changes_equivariantly}
    \end{align}
    \item and $\tilde{J}_\genv^{-1} \circ \tilde{L}_\genv \circ \tilde{J}_\genv$ combines to a valid left action:
    \begin{align}
        \forall \genv \in \Genv \;\exists \; \genv' \in \Genv: \tilde{J}_\genv^{-1} \circ \tilde{L}_\genv \circ \tilde{J}_\genv = \tilde{L}_{\genv'}.\label{equ:normal_left_action}
    \end{align}
    \end{enumerate}
\end{definition}
Alternatively, the above reads that 1) all tasks have ``similar'' symmetries and 2) their symmetries are related through compatible representations. The $2$-D navigation task indeed has compatible symmetry.

\begin{example}[Compatible Symmetries in $2$-D navigation]
    Let $\taskembedding = B \cdot \taskembedding_0$, $B \in \text{SO}(2, \mathbb{R})$ denote a task encoding from Example \ref{example:2_d_navigation}. Example \ref{example:symmetry_in_2_D_navigation} showed that a symmetry embedding $\encoderS_\gtask$ is given by $\encoderS_\gtask(s) = s-\taskembedding$. Define left actions $J_{\genv} \coloneqq B \cdot s$ and $\tilde{J}_{\genv}\coloneqq B \cdot \tilde{s}$. Then, \eqref{equ:encoder_changes_equivariantly} holds
    \begin{align}
        \tilde{J}_{\genv} \circ \encoderS_0 \circ J_{\genv}^{-1} \cdot s = B \cdot (B^{-1} \cdot s - \taskembedding_0)
        = s-\taskembedding = \encoderS_\gtask(s).\nonumber
    \end{align}
    Further, \eqref{equ:normal_left_action} holds as the product of rotation matrices is a rotation matrix.
\end{example}

Finally, we show that Compatible Symmetry gives rise to a class of systems with hereditary geometry. This is formalized in the following result:

\begin{theorem}[Hereditary Geometry from Symmetry]\label{theo:her_geo_from_symmetry}
    Let $(\mathcal{M}_0, \Gtask)$ be a geometric Meta-MDP that induces a task space $\mathbb{M}$ whose symmetries are compatible with $\Gtask$---in the sense of Def. \ref{def:compatible-Symmetry}. Let $\mathcal{M} \in \mathbb{M}$ denote a task induced by $\gtask \in \Gtask$ and assume that there exists $\genv \in \Genv$ such that
    \begin{equation}
          \pi^*\left(a \mid s; J_\gtask \cdot \taskembedding_0\right) =\pi^*(a \mid J_\genv^{-1} \encoderS_\gtask^{-1}\tilde{L}_\genv\encoderS_\gtask(s); \taskembedding_0). \label{equ:policy_push_forward}
    \end{equation}
    Then, the geometry in $\mathbb{M}$ is hereditary. 
\end{theorem}

\begin{proof}
    We have that:
    \begin{align}
        \pi^*(a \mid s; J_\gtask \cdot \taskembedding_0)\nonumber
        =\;&\pi^*(a \mid J_\genv^{-1} \encoderS_\gtask^{-1}\tilde{L}_\genv\encoderS_\gtask(s); \taskembedding_0)\nonumber\\
        =\;&\pi^*(a \mid J_\genv^{-1}J_\genv\encoderS_0^{-1}\tilde{J}_\genv^{-1}\tilde{L}_\genv\tilde{J}_\genv\encoderS_0J_\genv^{-1}\cdot s; \taskembedding_0)\nonumber\\
        =\;&\pi^*(a \mid \encoderS_0^{-1}\tilde{L}_{n'}\encoderS_0J_\genv^{-1} \cdot s; \taskembedding_0) \nonumber\\
        =\;& \encoderA_0^{-1}\tilde{K}_{n'}\encoderA_0\pi^*(a \mid J_\genv^{-1} \cdot s; \taskembedding_0)\label{equ:her_geo_dirty},
    \end{align}
    where the equalities respectively follow from \eqref{equ:policy_push_forward}, \eqref{equ:encoder_changes_equivariantly}, \eqref{equ:normal_left_action} and \eqref{equ:symmetry_equivariance}. Finally, define the Lie group $\Gnormal \coloneqq \Genv$. Then, $\Gnormal$ induces left actions $L_\gnormal: S \rightarrow S$ and $K_\gnormal: A \rightarrow A$ given by
    \begin{align}
        L_\gnormal \coloneqq J_\genv^{-1} \quad \text{and} \quad K_\gnormal \coloneqq \encoderA^{-1}_0\tilde{K}_{\genv'}\encoderA_0,\nonumber
    \end{align}
    where $\genv'$ is the normal left action from \eqref{equ:normal_left_action}. Plugging $L_\gnormal$ and $K_\gnormal$ into \eqref{equ:her_geo_dirty} and observing that $L_\gnormal$ and $K_\gnormal$ are linear left actions completes the proof. \qed
\end{proof}

Finally, the condition in \eqref{equ:policy_push_forward} reads that the change in the optimal policy for a new task (induced by $J_\gtask$) can be expressed with a symmetry transformation within $\mathcal{M}$ which is pulled back into $\mathcal{M}_0$ (via $J_\genv^{-1}$.)
We exemplify this below.
\begin{example}
    Equation \eqref{equ:policy_push_forward} holds in the $2$-D navigation task
    \begin{multline}
        \pi^*(a \mid J_{\genv}^{-1} \encoderS_\gtask^{-1} 
           \tilde{L}_{\genv} \encoderS_\gtask(s); \taskembedding_0)\\
           = \taskembedding_0 - \left[ B^{-1} \cdot (B(s-\taskembedding)  + \taskembedding) \right]\nonumber\\
          = (-1) \cdot (s - B \cdot \taskembedding_0)
          = \pi^*(a \mid s; J_\gtask \cdot \taskembedding_0).\nonumber
    \end{multline}
\end{example}

\section{Learning Problem} \label{sec:learning_problem}

Herein, we develop our learning problem, split into meta-train and meta-test. At meta-train time, we aim to learn the hereditary geometry $\Gnormal$, that is its left actions $L_\gnormal$ and $K_\gnormal$ and their representations $(\encoderS, \encoderA)$. At meta-test time, we aim to infer the left action of the test task. Adopting the model based approach from Lemma \ref{lemma:pushforward_mdp_implies_pushforward_optimal_policies}, we now show how to convert \eqref{equ:push_forward_reward_and_transition} into a differentiable optimization objective. The two main characteristics of our learning problem are that it is (i) symmetry based and (ii) defined on the differential of the underlying system. 

Fix a task $\mathcal{M}_0 \in \{\mathcal{M}_1,...,\mathcal{M}_{N_\text{train}}\}$ induced by $\gtask_0 \in \Gtask$ and assume w.l.o.g. that $\gtask_0 = e$.\footnote{Otherwise define a new group $\tilde{\Gtask}$ where each left action is shifted by $(\gtask_0)^{-1}$.} Given another task $\mathcal{M} \in \{\mathcal{M}_1,...,\mathcal{M}_{N_\text{train}}\}$ induced by some $\gtask \in \Gtask$, we now aim to discover left actions $L_\gnormal$ and $K_\gnormal$ such that for all $s \in S, a \in A$
\begin{align}
    R(s,a; J_\gtask \cdot \taskembedding_0) &= R(L_\gnormal \cdot s, K_\gnormal \cdot a; \taskembedding_0), \nonumber\\
    T(s' \mid s,a; J_\gtask \cdot \taskembedding_0) &= T(L_\gnormal \cdot s' \mid L_\gnormal \cdot s, K_\gnormal \cdot a; \taskembedding_0). \label{equ:push_forward_functional_T}
\end{align}
In the final step, we then learn the left actions $L_{\gnormal_{0:N_\text{tasks}}}$ and $K_{\gnormal_{0:N_\text{tasks}}}$ of all tasks independently and generators $W_S$ and $W_A$ of $L_\gnormal$ and $K_\gnormal$. (We define these generators momentarily.)  Mapping $(s,a) \mapsto (L_\gnormal^{-1} \cdot s, K_\gnormal^{-1} \cdot a)$, and recalling that $\gnormal \in \Gnormal \Leftrightarrow \gnormal^{-1} \in \Gnormal$,  the above equivalently reads

\begin{multline}
     R^\gtask(L_\gnormal \cdot s, K_\gnormal \cdot a) \coloneqq R(L_\gnormal \cdot s, K_\gnormal \cdot a; J_\gtask \cdot \taskembedding_0) = R(s,a; \taskembedding_0)\\
     = R(s,a; J_{\gtask_0} \cdot \taskembedding_0) 
       \eqqcolon R^0(s,a). \label{equ:push_forward_supervised_learning_R}
    \end{multline}
    \begin{multline}
    T^\gtask(L_\gnormal \cdot s' \mid L_\gnormal \cdot s, K_\gnormal \cdot a) \coloneqq T(L_\gnormal \cdot s' \mid L_\gnormal \cdot s, K_\gnormal \cdot a; J_\gtask \cdot \taskembedding_0) =\\
    T(s' \mid s,a; \taskembedding_0) = T(s' \mid s,a; J_{\gtask_0} \cdot \taskembedding_0)
       \eqqcolon T^0(s' \mid s,a). \label{equ:push_forward_supervised_learning_T}
\end{multline}
Given black box sampling access to $R^\gtask, R^0, T^\gtask$ and $T^0$, we could, in theory, approximate all functions with neural networks and minimize mean squared errors of  \eqref{equ:push_forward_supervised_learning_R} and \eqref{equ:push_forward_supervised_learning_T}. Instead of comparing the functions everywhere, we only compare the reward functions within their respective symmetries as we expect the hereditary geometry to arise from these. For the transition functions we use the approach described before and defer categorizing symmetries as in \eqref{equ:equivariance_T} to future work. To this end, we now aim to better understand the symmetries of $R$ to only search therein. Note that the symmetry of $R$ is precisely the conventional invariance statement.

\subsection{From functional to differential symmetries}

Let $\mathcal{M} = \{S,A,R,T,\gamma\}$ denote a task with symmetry $\Genv$. We now provide a necessary condition for the $\Genv$-invariance of $R$ via the $\mathfrak{\genv}$-invariance of its differential $dR$. $\mathfrak{\genv}$ denotes the Lie algebra of $\Genv$ which is the set of all smooth left-invariant vector fields on $\Genv$ that is isomorphic to the tangent space of $\Genv$ at the neutral element $e$ by the differential of left actions. This enables studying the symmetries induced by $\Genv$ by only considering tangent space at a single group element. In \S\ref{sec:simulations} we will show that this approach is thereby substantially more sample efficient and stable.

The Lie algebra $\mathfrak{\genv}$ is a $d$-dimensional vector space with basis $W=\{W^1,...,W^d\}$ where $d$ is the dimension of $\Genv$. Assuming a linear group structure in $\Genv$, the left action $L_\genv$ has a representation $\varrho_S: \mathfrak{n} \rightarrow \mathfrak{gl}(|S |, \mathbb{R})$ in the Lie algebra $\mathfrak{gl}(|S |, \mathbb{R})$ of the general linear group $GL(|S |, \mathbb{R})$. Matrices $\varrho_S(W^i) \coloneqq W_S^i \in \mathbb{R}^{\mid S \mid \times \mid S \mid}$ are called the infinitesimal generators of $L_\genv$ which we stack into a tensor $W_S \in \mathbb{R}^{d \times \mid S \mid \times \mid S \mid}$. If $\Genv$ is connected and compact, the set of all left actions $L_\genv$ for $\genv \in \Genv$ can be characterized with $W_S$ and the matrix-exponential exp (hence the name)
\begin{align}
    \bigcup_{\genv \in \Genv} \{L_\genv\} = \bigcup_{t \in \mathbb{R}^d}\exp(t \cdot W_S) \label{equ:generator},
\end{align}
where $\cdot$ denotes standard vector-tensor multiplication and similarly for $K_\genv$ and its generator $W_A \in \mathbb{R}^{d \times \mid A \mid \times \mid A \mid}$. Then, existing symmetry discovery approaches parameterize $W_S \in \mathbb{R}^{d \times \mid S \mid \times \mid S \mid}$ and $W_A \in \mathbb{R}^{d \times \mid A \mid \times \mid A \mid}$ and minimize a mean-squared-error type loss of \eqref{equ:invariance_R} by randomly sampling weights $t \in \mathbb{R}^d$ \cite{benton_learning_2020, yang_latent_2024}. 

We now show that one can forego such sampling and evaluate the invariance only at a single group element by considering the differential $dR$ of $R$. So, differential symmetries of $R$ are captured by its kernel distribution:
\begin{definition}[Kernel distribution]
    The \textit{kernel distribution} $D^R$ of $R$ contains the directional derivatives along $R$'s level sets
    \begin{align}
        D^R_{(s,a)} &\coloneqq \{v \in \mathbb{T}_{(s,a)}(S,A) \; | \;dR_{(s,a)}(v) = 0\} && (s,a) \in S \times A \nonumber\\
        D^R &\coloneqq \textstyle \bigcup_{(s,a) \in S \times A} D^R_{(s,a)}, \nonumber
    \end{align}
    where $\mathbb{T}_{s,a}(S,A)$ denotes the tangent space of $S$ and $A$ at $(s,a)$ which is, roughly speaking, the best linear approximation of $S \times A$ at $(s,a)$.
\end{definition}
Then, we have the following result:
\begin{lemma}[Differential symmetry implies functional symmetry]
    If the generators of $L_\genv$ and $K_\genv$ preserve $R$'s kernel distribution, that is for all $s \in S, a \in A, i=1,...,d$
    \begin{align}
        \begin{bmatrix}
       W_S^i & 0\\
       0 & W_A^i
   \end{bmatrix}\cdot (s,a)\in D^R_{(s,a)}, \label{equ:differential_invariance}
    \end{align}
    then $R$ is $\Genv$-invariant: $R(L_\genv s, K_\genv a) = R(s,a), \forall\genv\in \Genv$.
\end{lemma}

\begin{proof}
We begin by differentiating the invariance constraint in \eqref{equ:invariance_R} w.r.t. $(s,a)$. Then, we argue that the $\Genv$-invariance of $R$ only has to hold for group elements $\genv \approx e$ which are infinitesimally close to the identity $e$ as \eqref{equ:invariance_R} holds globally in $S \times A$. Thereby, $L_\genv$ and $K_\genv$ can be expressed with the generators $W_S$ and $W_A$ via $L_\genv = \text{Id} + c \cdot W_S$ and $K_\genv = \text{Id} + c \cdot W_A$ for some constant $c$ and $\text{Id}$ denotes the identity mapping. Assuming that the left action is linear, that is $W_S$ and $W_A$ are constant in $s$ and $a$, and using that the derivative operator is linear, one can then rewrite the differential of the invariance constraint into the stated claim.\qed
\end{proof}

Compare how the functional invariance in \eqref{equ:invariance_R} must be evaluated on $S \times A \times \Genv$ while the differential invariance in \eqref{equ:differential_invariance} only on $S \times A$. Note that the above provides a necessary condition that, as we will see, can be turned into a learning problem from samples. Even though this approach is very efficient, the sufficiency of these conditions must be studied which is deferred to our future work.

\subsection{A symmetry based approach for hereditary geometry discovery}

We now learn the hereditary geometry by only comparing $R^\gtask$ and $R^0$ along their differential symmetries, that is within the kernel distribution. This can be achieved by establishing the following necessary condition for the invariance property in \eqref{equ:push_forward_supervised_learning_R}. For \eqref{equ:push_forward_functional_T} we directly compare the two transition functions $T^\gtask$ and $T^0$.

\begin{lemma}\label{prop:symmetry_alignment}
    A necessary condition for  $R^\gtask(L_\gnormal \cdot s, K_\gnormal \cdot a) = R^0(s,a)$ to hold on $S\times A$ is that the differentials $dL_\gnormal$ and $dK_\gnormal$ of $L_\gnormal$ and $K_\gnormal$ push the differential symmetries $D^{R^0}$ of $R^0$ into the differential symmetries $D^{R^\gtask}$ of $R^\gtask$, that is for all $s\in S, a \in A, v \in E^{R^0}_{s,a}$
    \begin{align}\label{equ:push_forward_R_differential}
        d(L_\gnormal, K_\gnormal)_{s,a}[v] &\in \text{span}(E^{R^\gtask}_{L_\gnormal \cdot s, K_\gnormal \cdot a}),
    \end{align}
    where $E^R$ denotes a basis of $D^R$, also called a frame.
\end{lemma}
Note that as group actions are invertible, the above also implies that $L_\gnormal^{-1}, K_\gnormal^{-1}$ push the symmetries of $R^\gtask$ into the symmetries of $R^0$. Then, the $R$-symmetries of $\mathcal{M}_0$ and $\mathcal{M}$ are also called \textit{$\Gnormal$-related}. (cf. Page 182 in \cite{lee_introduction_2012}) and we see them as the same up to $\Gnormal$.

\begin{proof}[Proof of \Cref{prop:symmetry_alignment}]
The claim is a direct consequence of the following implications
    \begin{align*}
        &R^\gtask(L_\gnormal \cdot s, K_\gnormal \cdot a) = R^0(s,a) \\
        &\Rightarrow  dR^h_{L_\gnormal \cdot s, K_\gnormal \cdot a} \circ d(L_\gnormal, K_\gnormal)_{s,a}[v] = dR^0_{s,a}[v],  \forall v \in \mathbb{T}_sS \times \mathbb{T}_aA\\
        &\Rightarrow  dR^h_{L_\gnormal \cdot s, K_\gnormal \cdot a} \circ d(L_\gnormal, K_\gnormal)_{s,a}[v] = dR^0_{s,a}[v] = 0, \forall  v \in D^{R^0}_{s,a}\\
        &\Leftrightarrow  d(L_\gnormal, K_\gnormal)_{s,a}[v] \in D^{R^\gtask}_{L_\gnormal \cdot s, K_\gnormal \cdot a},  \forall v \in D^{R^0}_{s,a}\\
        &\Leftrightarrow  d(L_\gnormal, K_\gnormal)_{s,a}[v] \in \text{span}(E^{R^\gtask}_{L_\gnormal \cdot s, K_\gnormal \cdot a}),  \forall v \in E^{R^0}_{s,a},
    \end{align*}
    which hold for all $s \in S, a \in A$.\qed
\end{proof}

As $L_\gnormal$ and $K_\gnormal$ can be linearized by $\encoderS$ and $\encoderA$, \eqref{equ:push_forward_functional_T} and \eqref{equ:push_forward_supervised_learning_R} can be cast into an optimization problem with Objectives \ref{optimization_1_a} and \ref{optimization_1_b} as follows: Find $\tilde{L}_\gnormal \in \text{GL}(|S|, \mathbb{R})$ and $\tilde{K}_\gnormal \in \text{GL}(|A|, \mathbb{R})$ and diffeomorphisms (encoders/decoders) $\encoderS: S \rightarrow \tilde{S}$ and $\encoderA: A \rightarrow \tilde{A}$ (and correspondingly $\encoderS^{-1}: \tilde{S} \rightarrow S$ and $\encoderA^{-1}: \tilde{A} \rightarrow A$) such that for all $s,s'\in S, a \in A, v \in E^{R^0}_{s,a}$
\begin{align}
        d(L_\gnormal, K_\gnormal)_{s,a}[v] \in \text{span}(E^{R^\gtask}_{L_\gnormal \cdot s, K_\gnormal \cdot a}) \tag{{\textcolor{black}{O1a}}}\label{optimization_1_a}\\
        T^\gtask(L_\gnormal \cdot s' \mid L_\gnormal \cdot s, K_\gnormal \cdot a) = T^0(s' \mid s,a)\tag{{\textcolor{black}{O1b}}}\label{optimization_1_b}
s    \end{align}
where $\encoderS^{-1}\tilde{L}_\gnormal\encoderS \coloneqq L_\gnormal$ and $\encoderA^{-1}\tilde{K}_\gnormal\encoderA \coloneqq K_\gnormal$ to keep the notation consistent.

Finally, using \eqref{equ:generator}, learning generators $W_S$ of $\tilde{L}_\gnormal$ and $W_A$ of $\tilde{K}_\gnormal$ can be cast into another optimization with Objectives \ref{optimization_2_a} and \ref{optimization_2_b}: Learn differential generators $W_S \in \mathbb{R}^{d \times \mid S \mid \times \mid S \mid}$ and $W_A \in \mathbb{R}^{d \times \mid A \mid \times \mid A \mid}$ that contain the left actions $(\tilde{L}_\gnormal)_{1:N_\text{tasks}}$ and $(\tilde{K}_\gnormal)_{1:N_\text{tasks}}$:
\begin{align}
&(\tilde{L}_\gnormal)_{1:N_\text{tasks}} \in \exp(\text{span}(W_S)) \Leftrightarrow   \log((\tilde{L}_\gnormal)_{1:N_\text{tasks}}) \in \text{span}(W_S) \tag{{\textcolor{black}{O2a}}}\label{optimization_2_a} \\ 
&\text{and similarly} \nonumber\\ 
&\log((\tilde{K}_\gnormal)_{1:N_\text{tasks}}) \in \text{span}(W_A) \tag{{\textcolor{black}{O2b}}}\label{optimization_2_b}
\end{align}
where $\log$ denotes the matrix logarithm.

\subsection{Practical implementation of ({\textcolor{black}{O1}}) and ({\textcolor{black}{O2}})}

We smoothen the binary span constraints $a \in \text{span}(b)$ of tensor $a$ w.r.t. the vector space with basis $b$ from \eqref{optimization_1_a}, \eqref{optimization_2_a} and \eqref{optimization_2_b} by minimizing the orthogonal complement, denoted by $a_{b^\perp}$, yielding our loss functions
\begin{multline}
\mathcal{L}_{\text{geo}}^R(\tilde{L}_\gnormal, \tilde{K}_\gnormal, \encoderS, \encoderS^{-1},\encoderA, \encoderA^{-1} \mid E^{R^\gtask}, E^{R^0})=\\ 
     \mathbb{E}_{s,a\in S \times A} \Big[
    \sum_{v \in E^{R^0}(s,a)}
        \left\|\big[d(L_\gnormal, K_\gnormal)_{s,a}[v]\big]_{(E^{R^\gtask}
    (L_\gnormal \cdot s, K_\gnormal \cdot a))^\perp}\right\|_2 \\
    + \lambda_1( \big\|\encoderS^{-1}\encoderS(s)-s\big\|_2 + \big\|\encoderA^{-1}\encoderA(a)-a\big\|_2)
    \Big] \\+ \lambda_0 
         \|\encoderS, \encoderS^{-1}, \encoderA, \encoderA^{-1}\|_1 , \label{equ:loss_geo_R}
    \end{multline}
    \begin{multline}
    \mathcal{L}_\text{geo}^T(\tilde{L}_\gnormal, \tilde{K}_\gnormal, \encoderS, \encoderS^{-1}, \encoderA, \encoderA^{-1} \mid T^\gtask, T^0) =\\
    \mathbb{E}_{s, a \in S \times A} ||T^\gtask(L_\gnormal \cdot s' \mid L_\gnormal \cdot s, K_\gnormal \cdot a) - T^0(s' \mid s,a)|| ,\label{equ:loss_geo_T}
    \end{multline}
    \begin{multline}
    \mathcal{L}_{\text{gen}}(W_\square \mid \log(\square_\gnormal)) \\
    = \|\log(\tilde{\square}_\gnormal)_{(W_\square/\|W_\square\|_F)^\perp}\|_F
    \;+\; \lambda_0 \|W_\square\|_1 , \label{equ:loss_gen}
\end{multline}
where the sum of $v \in E^{R^0}(s,a)$ ranges over the basis vectors from the frame $E^{R^0}(s,a)$, $W_\square$ denotes a generator (which we set to $W_S$ or $W_A$) and $\tilde{\square}_\gnormal$ a linear left action induced by $W^\square$ (which we set to $L_\gnormal$ or $K_\gnormal$), $\lambda_{0:1}>0$ weights and $\|\cdot\|_F$ the Frobenius norm. Additionally, we added Lasso regularizers on the weights of all decision variables and a reconstruction regularizer for the representations $\encoderS$ and $\encoderA$. 

To simplify the optimization, we trivialize the constrained domain $\tilde{L}_\gnormal \in \text{GL}(|S|, \mathbb{R})\subsetneq \mathbb{R}^{| S | \times | S |}$ by learning the differential $\log(\tilde{L}_\gnormal) \in \mathfrak{gl}(| S |,\mathbb{R})=\mathbb{R}^{| S | \times | S |}$ and similarly for $\tilde{K}_\gnormal$. As a practical benefit, it is far easier to compute the exponential of a matrix than its logarithm. (See \cite{contributors_pytorch_2018} for a discussion.) Finally, we sample from the replay buffer using a Boltzmann distribution with a tuneable temperature to sample more uniformly from the support of the occupation measure.

We now state our final learning problem. Given estimates $E^{R^0},...,E^{R^{N_\text{train}}}$ and $T^0,...,T^{N_\text{train}}$ of the reward function kernels and transition functions, we minimize the sum of the three loss functions \eqref{equ:loss_geo_R}, \eqref{equ:loss_geo_T} and \eqref{equ:loss_gen}
\begin{gather}
    \min_{\text{data}_\text{geo}, W_S, W_A} \mathcal{L}_\text{gen}(W_S \mid \log L_{\gnormal_i}) + \mathcal{L}_\text{gen}(W_A \mid \log K_{\gnormal_i})\nonumber\\
    +\sum_{i=1}^{N_\text{tasks}} \mathcal{L}_\text{geo}^R(\text{data}_\text{geo} \mid E^{R_0}, E^{R_i})
    + \mathcal{L}_\text{geo}^T(\text{data}_\text{geo} \mid T^{\gtask^i}, T^0)\nonumber
\end{gather}
where $\text{data}_\text{geo}=(\log \tilde{L}_{\gnormal_i},\log  \tilde{K}_{\gnormal_i}, \encoderS, \encoderS^{-1},\encoderA, \encoderA^{-1})$ and we convert $\log \tilde{L}_{\gnormal_i}$ and $\log \tilde{K}_{\gnormal_i}$ into $\tilde{L}_{\gnormal_i}$ and $\tilde{K}_{\gnormal_i}$ via the matrix exponential.

\subsection{Meta-test} 

At meta-test time, we sample a replay buffer in the test task $\mathcal{M}$ according to a random policy to estimate the kernel and transition function of $\mathcal{M}$. Then, we minimize
\begin{gather}
    \min_{c \in \mathbb{R}^d}  \mathcal{L}_{\text{geo}}^R(L_\gnormal, K_\gnormal \mid \text{data}_\text{geo}) + \mathcal{L}_{\text{geo}}^T(L_\gnormal, K_\gnormal \mid \text{data}_\text{geo})
\end{gather}
where $L_\gnormal = \exp(c \cdot W_S) \text{ and } K_\gnormal = \exp(c \cdot W_A)$ and data$_\text{geo} = (W_S, W_A, \encoderS, \encoderS^{-1}, \encoderA, \encoderA^{-1}, E^R, E^{R^0}, T^\gtask, T^0)$. Note that we can use the geometric structure and its representation discovered at the training time, only searching within the spans of the generators $W_S$ and $W_A$.

\section{Empirical evaluation}\label{sec:simulations}

\begin{figure}[!t]
\centering
\includegraphics[width=\linewidth]{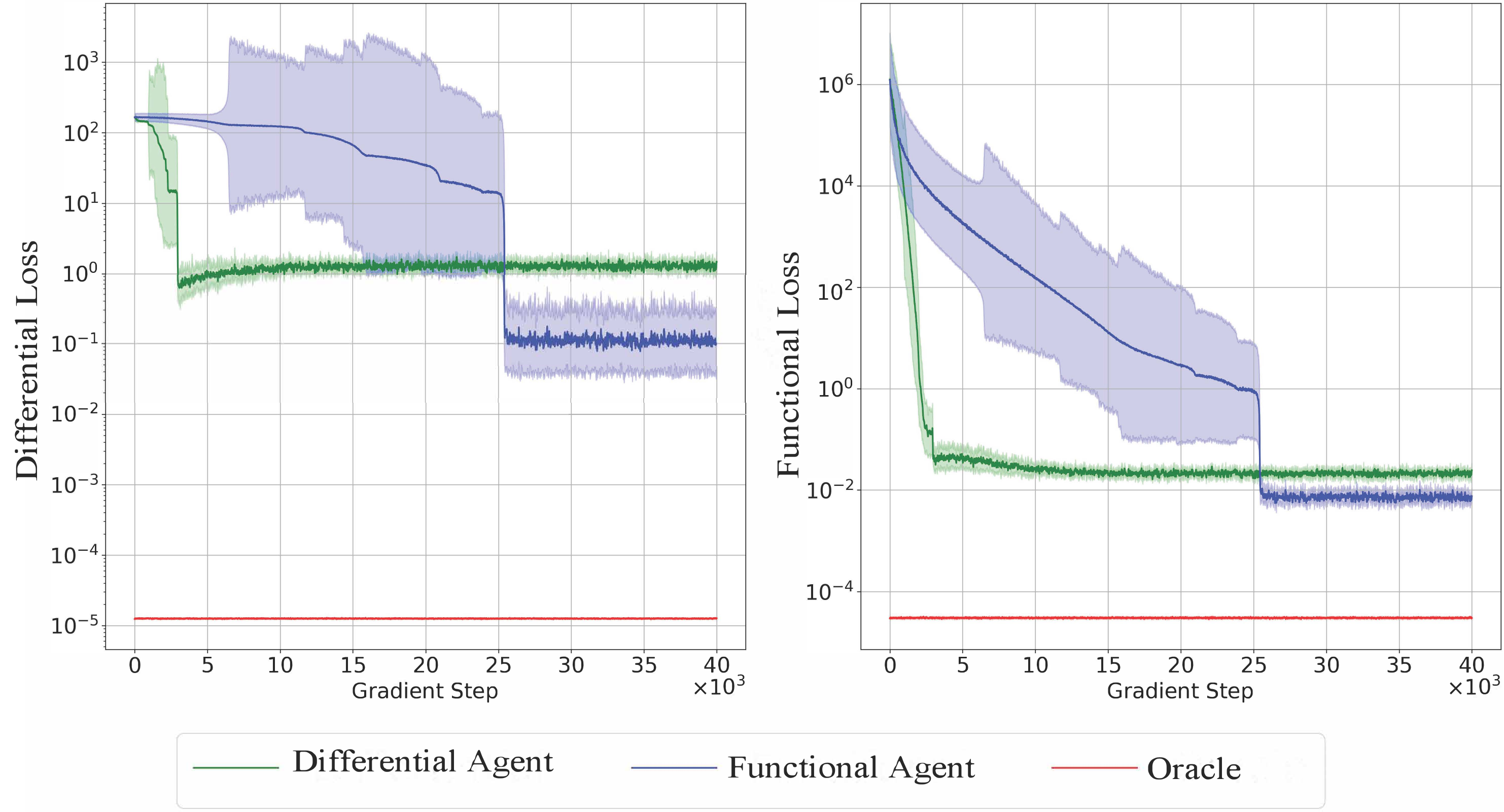}
\caption{Differential \textcolor{ForestGreen}{(green)} and Functional \textcolor{blue}{(blue)} symmetry discovery agents against an Oracle \textcolor{red}{(red)} evaluating either the differential loss (left) or the functional loss (right) over time---lower is better. The differential symmetry discovery (green) is an order of magnitude more sample efficient and stable.}
\label{fig:results_symmetry_discovery}
\vspace{-0.4cm}
\end{figure}

The following questions are addressed to empirically validate our approach: (a) how a kernel frame $E^R$ can be learned from samples, (b) how differential and functional symmetry discovery compare, and (c) whether geometric Meta-RL can improve generalization.

As a test bench, we use the $2$-D navigation task in \Cref{example:2_d_navigation}. Concerning (a), we adopt a simple numerical differentiation approach.
Concerning (b), we compare the sample efficiency of our proposed differential symmetry discovery to \textit{Augerino} \cite{benton_learning_2020} which minimize the functional symmetry constraint in \eqref{equ:invariance_R}. Concerning (c), we uniformly sample only 4 tasks $\mathcal{M} \sim U(\mathbb{M})$ at train time---for an emphasis on the fact that the meta-dataset is not necessarily densely sampled from the task space. After an exploration period at test time, we roll-out the agent without further training. Then, we compute the regret against an optimal policy in the test task and plot it against the distance to the closest training task. We compare our approach against \textit{Contrastive learning augmented Context-based Meta-RL (CCM)} \cite{fu_towards_2021} which combines contrastive learning for task encoding and SAC for policy training.

We begin by learning a frame of a kernel distribution. Let $\mathcal{B}=\{s_i,a_i,R(s_i,a_i), T(s_i,a_i)\}_{i=1}^N$ denote a replay buffer of some training task and $\epsilon_1 > 0$ denote a hyperparameter that determines when $(s,a)$ and $(s',a')$ are "close". Then, compute for each sample $(s,a) \in \mathcal{B}$ its neighbors $\mathcal{N}(s,a)$:
\begin{align}
    \mathcal{N}(s,a) = \{ (s',a') \in \mathcal{B} \mid ||(s,a)-(s',a')||_2 \leq \epsilon_1 \}\nonumber
\end{align}
Fix $\epsilon_2>0$, and define the local level set $\hat{L}_{(s,a)}$ of a point $(s,a) \in S \times A$:
\begin{align}
    \hat{L}_{(s,a)} &\coloneqq \{(s',a') \in \mathcal{N}(s,a) \; | \; \bigl\|R(s',a') - R(s,a)\bigl\|_2 \leq \epsilon_2\} \nonumber
\end{align}
By Taylor expansion of $R$ around $(s,a) \in \mathcal{B}$, we have $R(s',a') \approx R(s,a) + dR_{(s,a)}(s'-s,a'-a)$ so $(s'-s,a'-a) \in D^R((s,a))$ for $(s',a') \in \hat{L}_{(s,a)}$. An orthonormal basis $E^R(s,a)$ for each $(s,a)$ can then be computed via PCA on $(s'-s,a'-a)$ for all $(s',a') \in \hat{L}_{(s,a)}$. Finally, we smoothly connect the point-wise bases $E^R(s,a)$ for all $(s,a) \in \mathcal{B}$ into a frames $E^R$ via a dense neural network.

Next, we define the \textit{differential loss} $\mathcal{L}_\text{diff}$ to be the 2-norm of the orthogonal complement of \eqref{equ:differential_invariance}:
\begin{multline}
\mathcal{L}_{\text{diff}}(W_S, W_A \mid \mathcal{B}, E^R)
\coloneqq \nonumber\\
\mathbb{E}_{(s,a) \sim \mathcal{B}} 
\left[
\sum_{i=1}^d
\left\|
\left(
\begin{bmatrix}
W_S^i & 0 \\
0 & W_A^i
\end{bmatrix}
\begin{bmatrix}
s \\
a
\end{bmatrix}
\right)_{(E^R{(s,a)})^\perp}
\right\|_2
\right], \nonumber
\end{multline}
and the \textit{functional loss} $\mathcal{L}_{\text{func}}$ as the mean squared error of \eqref{equ:invariance_R}:
\begin{align}
    \mathcal{L}_{\text{func}}(W_S, W_A \mid \mathcal{B})
\coloneqq \mathbb{E}_{(s,a) \sim \mathcal{B}} \mathbb{E}_{t \sim \mathcal{N}(0_{d \times 1},1_{d \times d})} \nonumber\\ [||R(\exp(t \cdot W_S) \cdot s, \exp(t \cdot W_A) \cdot a; \taskembedding) - R(s,a; \taskembedding)||_2], \nonumber
\end{align}
where "time"-constants $t \in \mathbb{R}^d$ are sampled to generate different left actions. (See \cite{yang_latent_2024} for further discussion.) Then, Figure \ref{fig:results_symmetry_discovery} shows loss over time (lower is better) of the differential (green) and the functional (blue) symmetry discovery agents, evaluated with the differential loss (left) or the functional loss (right)---with results averaged across 10 runs with identical initializations. In both cases, differential symmetry discovery agent converges an order of magnitude faster (2.5 vs. 25k steps) with lower variance. The overall final lower loss of the functional approach did not result in any meaningful changes in the learned generator; Both methods converged against the ground-truth $\text{SO}(2,\mathbb{R})$ symmetry.

\begin{figure}[!t]
\centering
\includegraphics[width=\linewidth]{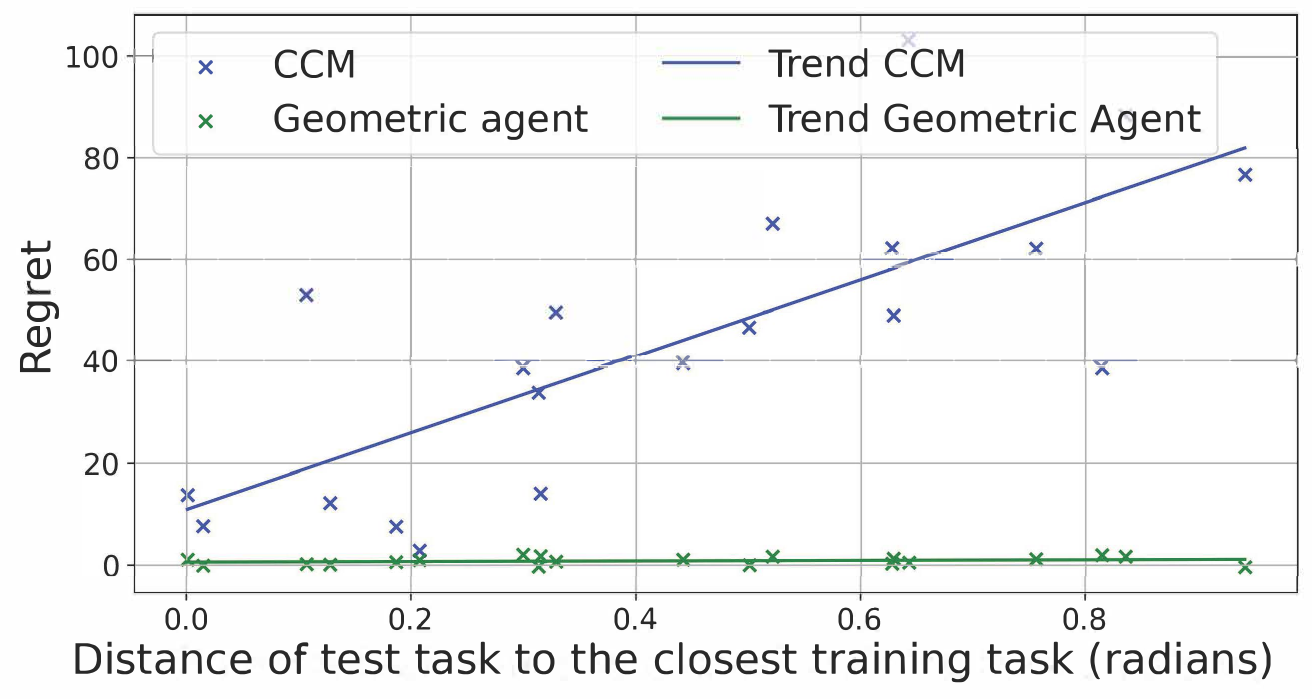}
\caption{Generalization in the $2$-D navigation task: The regret versus the distance to closest training task for our geometric approach \textcolor{ForestGreen}{(green)} and CCM \textcolor{blue}{(blue)}---lower is better. The CCM agent generalizes well to nearby tasks but collapses for tasks distant to the training set while the geometric agent generalizes within the entire task space.}
\label{fig:results_generalization}
\vspace{-0.4cm}
\end{figure}

Finally, Figure \ref{fig:results_generalization} shows the regret $\mathcal{R(\pi; \mathcal{M})}$ as a function of distance to closest training task (lower is better) for our proposed geometric approach (in green) versus CCM (in blue). CCM generalizes by local smoothness where the regret increases monotonically in distance to the closest training task. In contrast, the geometric agent can generalize within the entire task space. The code-base for these simulations can be found at \url{https://github.com/PaulNitschke/Hereditary-Geometries}.

\section{Conclusions}\label{sec:conclusions}

In this paper we studied non-local generalization in Meta-RL via symmetries which we formalized as a hereditary geometry. We argued that hereditary geometries commonly arise from the symmetries of the system and proposed a symmetry-based learning problem that is empirically validated in a $2$-D navigation task. We conclude this paper by discussing some shortcomings and potential future work.
First, in this study we assumed a model based perspective to discover $\Gnormal$ which only provides a sufficient but not a necessary requirement, that is we may converge to sub-symmetries. Second, even though in the implementation we only consider invariance for simplicity (which is sufficient for the 2-D Navigation task,) we emphasize that the policy generalization by \emph{equivariance} may be considered similarly for an effective outcome in more general settings.

\bibliographystyle{IEEEtran}
\bibliography{references2}

\begin{thebibliography}{10}
\providecommand{\url}[1]{#1}
\csname url@samestyle\endcsname
\providecommand{\newblock}{\relax}
\providecommand{\bibinfo}[2]{#2}
\providecommand{\BIBentrySTDinterwordspacing}{\spaceskip=0pt\relax}
\providecommand{\BIBentryALTinterwordstretchfactor}{4}
\providecommand{\BIBentryALTinterwordspacing}{\spaceskip=\fontdimen2\font plus
\BIBentryALTinterwordstretchfactor\fontdimen3\font minus
  \fontdimen4\font\relax}
\providecommand{\BIBforeignlanguage}[2]{{%
\expandafter\ifx\csname l@#1\endcsname\relax
\typeout{** WARNING: IEEEtran.bst: No hyphenation pattern has been}%
\typeout{** loaded for the language `#1'. Using the pattern for}%
\typeout{** the default language instead.}%
\else
\language=\csname l@#1\endcsname
\fi
#2}}
\providecommand{\BIBdecl}{\relax}
\BIBdecl

\bibitem{kirk_survey_2023}
R.~Kirk, A.~Zhang, E.~Grefenstette, and T.~Rocktäschel, ``A {Survey} of
  {Zero}-shot {Generalisation} in {Deep} {Reinforcement} {Learning},'' \emph{J.
  Artif. Int. Res.}, vol.~76, May 2023.

\bibitem{wang_learning_2017}
J.~X. Wang, Z.~Kurth-Nelson, D.~Tirumala, H.~Soyer, J.~Z. Leibo, R.~Munos,
  C.~Blundell, D.~Kumaran, and M.~Botvinick, ``Learning to reinforcement
  learn,'' Jan. 2017, arXiv:1611.05763 [cs].

\bibitem{rakelly_efficient_2019}
K.~Rakelly, A.~Zhou, C.~Finn, S.~Levine, and D.~Quillen,
  ``\BIBforeignlanguage{en}{Efficient {Off}-{Policy} {Meta}-{Reinforcement}
  {Learning} via {Probabilistic} {Context} {Variables}},'' in
  \emph{\BIBforeignlanguage{en}{Proceedings of the 36th {International}
  {Conference} on {Machine} {Learning}}}.\hskip 1em plus 0.5em minus
  0.4em\relax PMLR, May 2019, pp. 5331--5340.

\bibitem{mandi_effectiveness_2022}
Z.~Mandi, P.~Abbeel, and S.~James, ``\BIBforeignlanguage{en}{On the
  {Effectiveness} of {Fine}-tuning {Versus} {Meta}-reinforcement {Learning}},''
  \emph{\BIBforeignlanguage{en}{Advances in Neural Information Processing
  Systems}}, 2022.

\bibitem{li_towards_2024}
L.~Li, H.~Zhang, X.~Zhang, S.~Zhu, Y.~Yu, J.~Zhao, and P.-A. Heng,
  ``\BIBforeignlanguage{en}{Towards an {Information} {Theoretic} {Framework} of
  {Context}-{Based} {Offline} {Meta}-{Reinforcement} {Learning}},''
  \emph{\BIBforeignlanguage{en}{Advances in Neural Information Processing
  Systems}}, 2024.

\bibitem{li_focal_2020}
L.~Li, R.~Yang, and D.~Luo, ``\BIBforeignlanguage{en}{{FOCAL}: {Efficient}
  {Fully}-{Offline} {Meta}-{Reinforcement} {Learning} via {Distance} {Metric}
  {Learning} and {Behavior} {Regularization}},'' in
  \emph{\BIBforeignlanguage{en}{Int. {Conf}. on {Learning} {Representations}}},
  Oct. 2020.

\bibitem{zintgraf_varibad_2021}
L.~Zintgraf, S.~Schulze, C.~Lu, L.~Feng, M.~Igl, K.~Shiarlis, Y.~Gal,
  K.~Hofmann, and S.~Whiteson, ``{VariBAD}: variational {Bayes}-adaptive deep
  {RL} via meta-learning,'' \emph{J. Mach. Learn. Res.}, Jan. 2021.

\bibitem{graf_dissecting_2021}
F.~Graf, C.~Hofer, M.~Niethammer, and R.~Kwitt, ``Dissecting {Supervised}
  {Contrastive} {Learning},'' in \emph{Proceedings of the 38th {International}
  {Conference} on {Machine} {Learning}}.\hskip 1em plus 0.5em minus 0.4em\relax
  PMLR, Jul. 2021.

\bibitem{damrich_umap_2021}
S.~Damrich and F.~A. Hamprecht, ``On {UMAP}' s {True} {Loss} {Function},'' in
  \emph{Advances in {Neural} {Information} {Processing} {Systems}}, 2021.

\bibitem{bohm_attraction-repulsion_2022}
J.~N. Böhm, P.~Berens, and D.~Kobak,
  ``\BIBforeignlanguage{en}{Attraction-{Repulsion} {Spectrum} in {Neighbor}
  {Embeddings}},'' \emph{\BIBforeignlanguage{en}{Journal of Machine Learning
  Research 23 (2022) 1-32}}, 2022.

\bibitem{lee_parameterizing_2023}
S.~Lee, M.~Cho, and Y.~Sung, ``Parameterizing {Non}-{Parametric}
  {Meta}-{Reinforcement} {Learning} {Tasks} via {Subtask} {Decomposition},''
  \emph{Advances in Neural Information Processing Systems}, vol.~36, 2023.

\bibitem{schulman_proximal_nodate}
J.~Schulman, F.~Wolski, P.~Dhariwal, A.~Radford, and O.~Klimov, ``Proximal
  {Policy} {Optimization} {Algorithms},'' arXiv:1707.06347 [cs].

\bibitem{haarnoja_soft_2018}
T.~Haarnoja, A.~Zhou, P.~Abbeel, and S.~Levine, ``Soft {Actor}-{Critic}:
  {Off}-{Policy} {Maximum} {Entropy} {Deep} {Reinforcement} {Learning} with a
  {Stochastic} {Actor},'' Aug. 2018, arXiv:1801.01290 [cs].

\bibitem{pal_soft_2000}
S.~K. Pal, T.~S. Dillon, and D.~S. Yeung, Eds., \emph{Soft computing in case
  based reasoning}.\hskip 1em plus 0.5em minus 0.4em\relax Berlin, Heidelberg:
  Springer-Verlag, Sep. 2000.

\bibitem{lee_introduction_2012}
J.~M. Lee, \emph{Introduction to {Smooth} {Manifolds}}, ser. Graduate {Texts}
  in {Mathematics}.\hskip 1em plus 0.5em minus 0.4em\relax Springer-Verlag,
  2012, vol. 218.

\bibitem{schuller_geometric_2015}
\BIBentryALTinterwordspacing
F.~Schuller, ``Geometric {Analysis} of {Theoretical} {Physics},'' 2015.
  [Online]. Available:
  \url{https://www.youtube.com/playlist?list=PLPH7f_7ZlzxTi6kS4vCmv4ZKm9u8g5yic}
\BIBentrySTDinterwordspacing

\bibitem{beck_survey_2024}
J.~Beck, R.~Vuorio, E.~Z. Liu, Z.~Xiong, L.~Zintgraf, C.~Finn, and S.~Whiteson,
  ``A {Survey} of {Meta}-{Reinforcement} {Learning},'' Aug. 2024,
  arXiv:2301.08028 [cs].

\bibitem{doshi-velez_hidden_2013}
F.~Doshi-Velez and G.~Konidaris, ``Hidden {Parameter} {Markov} {Decision}
  {Processes}: {A} {Semiparametric} {Regression} {Approach} for {Discovering}
  {Latent} {Task} {Parametrizations},'' Aug. 2013, arXiv:1308.3513 [cs].

\bibitem{hallak_contextual_2015}
A.~Hallak, D.~D. Castro, and S.~Mannor, ``Contextual {Markov} {Decision}
  {Processes},'' Feb. 2015, arXiv:1502.02259 [stat].

\bibitem{lee_improving_2021}
S.~Lee and S.-Y. Chung, ``\BIBforeignlanguage{en}{Improving {Generalization} in
  {Meta}-{RL} with {Imaginary} {Tasks} from {Latent} {Dynamics} {Mixture}},''
  in \emph{\BIBforeignlanguage{en}{Advances in {Neural} {Information}
  {Processing} {Systems}}}, Nov. 2021.

\bibitem{fu_performance_2022}
H.~Fu, J.~Yao, O.~Gottesman, F.~Doshi-Velez, and G.~Konidaris, ``Performance
  {Bounds} for {Model} and {Policy} {Transfer} in {Hidden}-parameter {MDPs},''
  in \emph{The {Eleventh} {International} {Conference} on {Learning}
  {Representations}}, Sep. 2022.

\bibitem{van_der_pol_mdp_2020}
E.~van~der Pol, D.~Worrall, H.~van Hoof, F.~Oliehoek, and M.~Welling, ``{MDP}
  {Homomorphic} {Networks}: {Group} {Symmetries} in {Reinforcement}
  {Learning},'' in \emph{Advances in {Neural} {Information} {Processing}
  {Systems}}, vol.~33.\hskip 1em plus 0.5em minus 0.4em\relax Curran
  Associates, Inc., 2020, pp. 4199--4210.

\bibitem{benton_learning_2020}
G.~Benton, M.~Finzi, P.~Izmailov, and A.~G. Wilson, ``Learning {Invariances} in
  {Neural} {Networks} from {Training} {Data},'' in \emph{Advances in {Neural}
  {Information} {Processing} {Systems}}, vol.~33.\hskip 1em plus 0.5em minus
  0.4em\relax Curran Associates, Inc., 2020, pp. 17\,605--17\,616.

\bibitem{yang_latent_2024}
J.~Yang, N.~Dehmamy, R.~Walters, and R.~Yu, ``Latent space symmetry
  discovery,'' in \emph{Proceedings of the 41st {International} {Conference} on
  {Machine} {Learning}}, ser. {ICML}'24, vol. 235.\hskip 1em plus 0.5em minus
  0.4em\relax Vienna, Austria: JMLR.org, Jul. 2024, pp. 56\,047--56\,070.

\bibitem{sutton_reinforcement_2018}
R.~Sutton and A.~Barto, \emph{Reinforcement {Learning}: {An} {Introduction}},
  2nd~ed.\hskip 1em plus 0.5em minus 0.4em\relax Cambridge: MIT press, 2018.

\bibitem{anderson_more_1972}
P.~W. Anderson, ``More {Is} {Different},'' \emph{Science}, vol. 177, no. 4047,
  pp. 393--396, Aug. 1972.

\bibitem{contributors_pytorch_2018}
\BIBentryALTinterwordspacing
Contributors, ``{PyTorch} {Issues}: [feature request] {Add} matrix functions
  {\textbackslash}\#9983,'' Jul. 2018. [Online]. Available:
  \url{https://github.com/pytorch/pytorch/issues/9983}
\BIBentrySTDinterwordspacing

\bibitem{fu_towards_2021}
H.~Fu, H.~Tang, J.~Hao, C.~Chen, X.~Feng, D.~Li, and W.~Liu, ``Towards
  {Effective} {Context} for {Meta}-{Reinforcement} {Learning}: an {Approach}
  based on {Contrastive} {Learning},'' in \emph{Proceedings of the {AAAI}
  {Conference} on {Artificial} {Intelligence}}, vol.~35, May 2021, pp.
  7457--7465.

\end{thebibliography}

\end{document}